\documentclass[10pt,twocolumn,letterpaper]{article}

\usepackage{cvpr}              

\usepackage{graphicx}	
\usepackage{amsmath}	
\usepackage{amssymb}	
\usepackage{booktabs}
\usepackage{times}
\usepackage{microtype}
\usepackage{marvosym}
\usepackage{epsfig}
\usepackage{caption}
\usepackage{float}
\usepackage{placeins}
\usepackage{color, colortbl}
\usepackage{xcolor}
\usepackage{soul}
\usepackage{stfloats}
\usepackage{enumitem}
\usepackage{tabularx}
\usepackage{xstring}
\usepackage{multirow}
\usepackage{xspace}
\usepackage{url}
\usepackage{subcaption}
\usepackage[hang,flushmargin]{footmisc}
\usepackage{bm}
\usepackage{amsthm}
\usepackage{algorithm}
\usepackage{algpseudocode}
\usepackage{pgfplots}

\definecolor{cvprblue}{rgb}{0.21,0.49,0.74}
\usepackage[pagebackref,breaklinks,colorlinks,allcolors=cvprblue]{hyperref}

\def\paperID{3665} 
\def\confName{CVPR}
\def\confYear{2025}

\title{FIMA-Q: Post-Training Quantization for Vision Transformers by Fisher Information Matrix Approximation}

\author{Zhuguanyu Wu$^{1,2 *}$, 
    Shihe Wang$^{1,2 *}$, 
    Jiayi Zhang$^{1,2}$,  
    Jiaxin Chen$^{1,2}$\textsuperscript{\Letter}, Yunhong Wang$^{1,2}$\textsuperscript{\Letter} \\
    $^1$State Key Laboratory of Virtual Reality Technology and Systems, Beihang University, China \\
    $^2$School of Computer Science and Engineering, Beihang University, Beijing, China \\
    {\tt\small \{goatwu, shihewang, zhangjyi, jiaxinchen, yhwang\}@buaa.edu.cn}
}

\newtheorem{theory}{Theorem}[section]
\newtheorem{definition}{Definition}

\newtheorem{corollary}{Corollary}[section]
\pgfplotsset{compat=1.18}
\begin{document}
\maketitle
\renewcommand{\thefootnote}{}

\begin{abstract}

Post-training quantization (PTQ) has stood out as a cost-effective and promising model compression paradigm in recent years, as it avoids computationally intensive model retraining. Nevertheless, current PTQ methods for Vision Transformers (ViTs) still suffer from significant accuracy degradation, especially under low-bit quantization. To address these shortcomings, we analyze the prevailing Hessian-guided quantization loss, and uncover certain limitations of conventional Hessian approximations. By following the block-wise reconstruction framework, we propose a novel PTQ method for ViTs, dubbed FIMA-Q. Specifically, we firstly establish the connection between KL divergence and FIM, which enables fast computation of the quantization loss during reconstruction. We further propose an efficient FIM approximation method, namely DPLR-FIM, by employing the diagonal plus low-rank principle, and formulate the ultimate quantization loss. Our extensive experiments, conducted across various vision tasks with representative ViT-based architectures on public datasets, demonstrate that our method substantially promotes the accuracy compared to the state-of-the-art approaches, especially in the case of low-bit quantization. The source code is available at \url{https://github.com/ShiheWang/FIMA-Q}.
\end{abstract}  
\footnote{\ $^*$\  Equal contribution.}
\footnote{\textsuperscript{\Letter} Corresponding author.}
\renewcommand{\thefootnote}{\arabic{footnote}}
\addtocounter{footnote}{-2}

\section{Introduction}
\label{sec:intro}

\begin{figure}[t]
    \centering
    \begin{subfigure}{.49\columnwidth}
        \centering
        \includegraphics[width=1\linewidth]{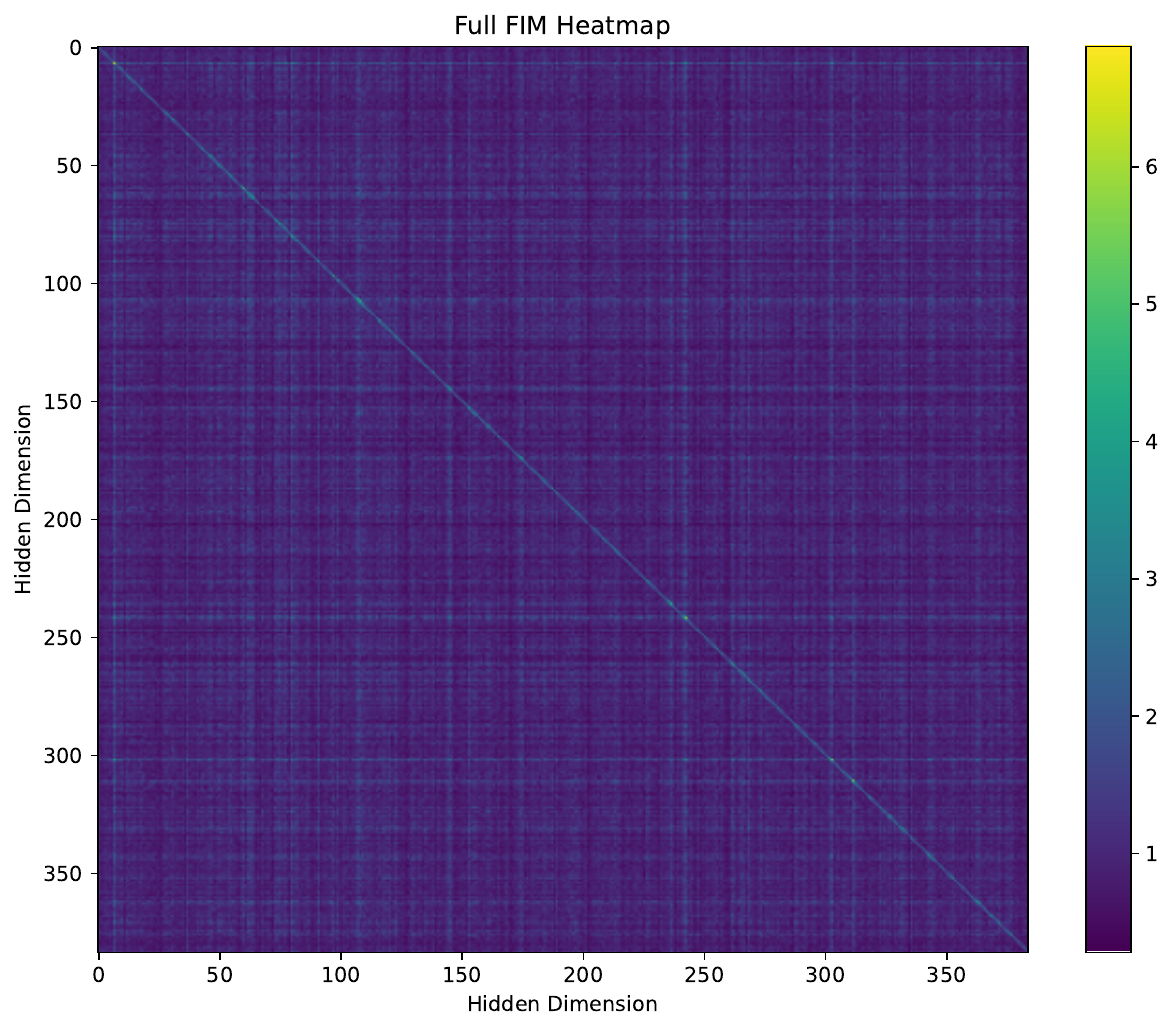}
        \caption{Complete FIM}
        \label{fig:fim_sub1}
    \end{subfigure}%
    \hfill
    \begin{subfigure}{.49\columnwidth}
        \centering
        \includegraphics[width=1\linewidth]{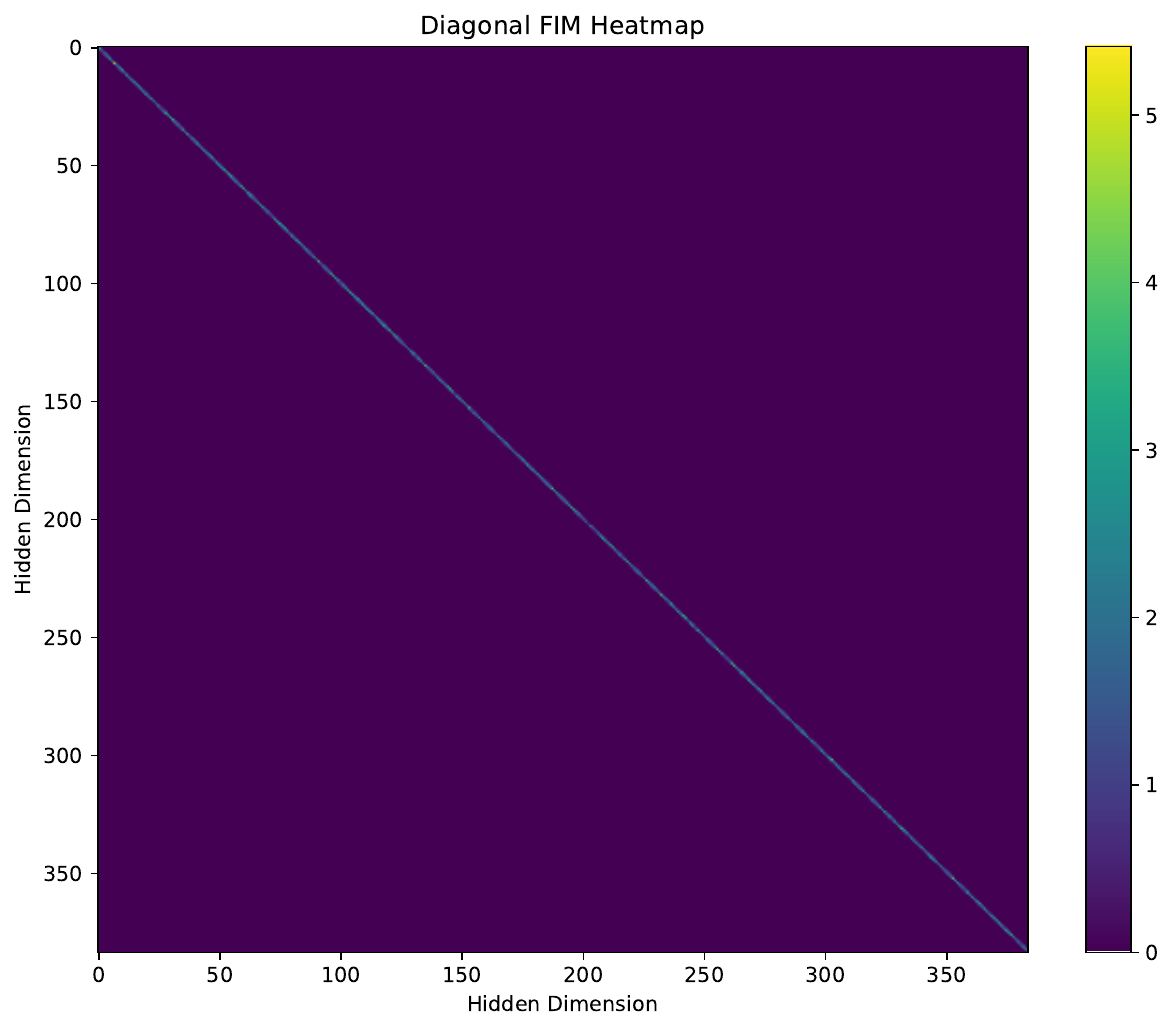}
        \caption{Diagonal FIM}
        \label{fig:fim_sub2}
    \end{subfigure}%
    \hfill
    \begin{subfigure}{.49\columnwidth}
        \centering
        \includegraphics[width=1\linewidth]{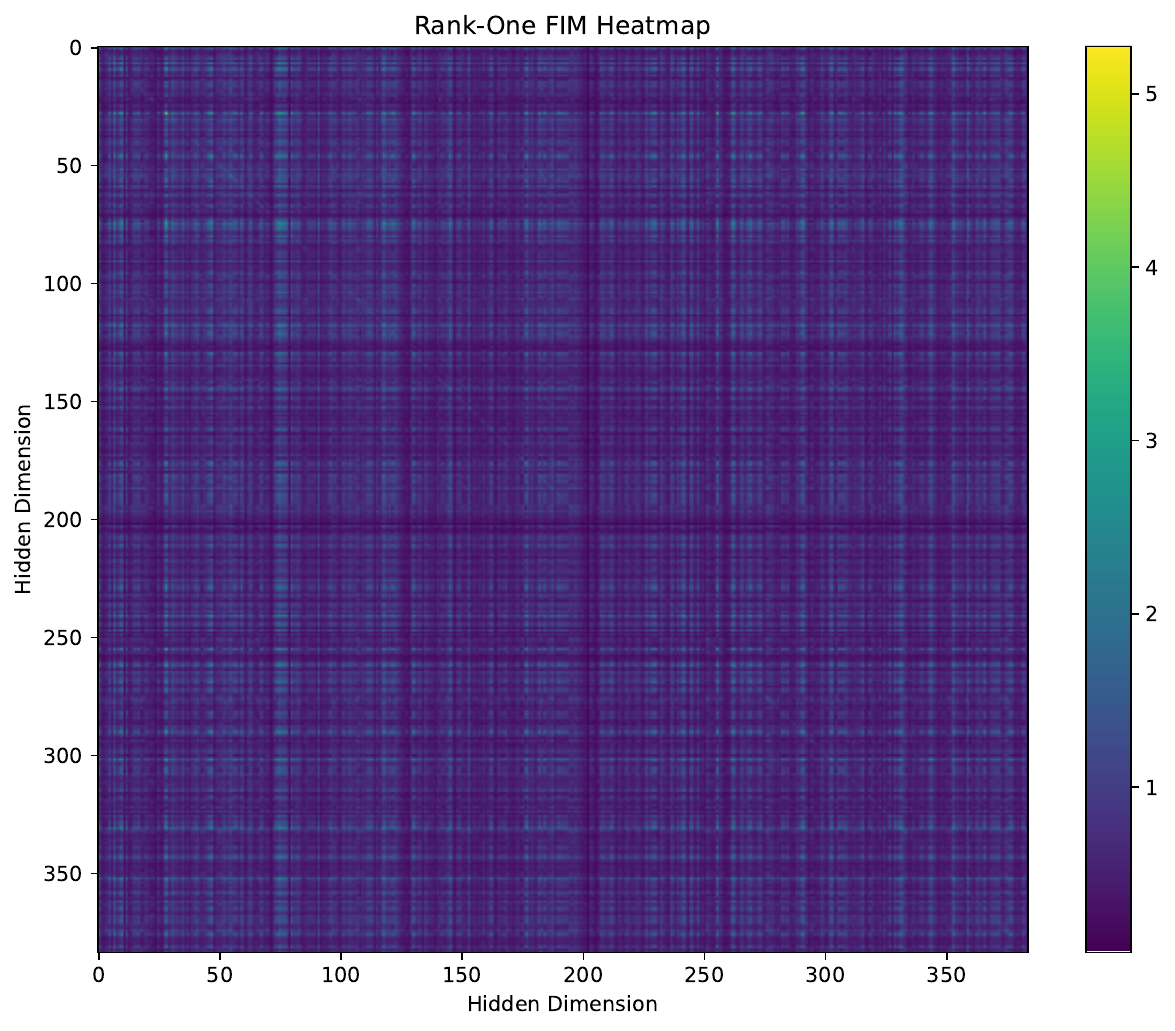}
        \caption{Low-Rank FIM}
        \label{fig:fim_sub3}
    \end{subfigure}
    \hfill
    \begin{subfigure}{.49\columnwidth}
        \centering
        \includegraphics[width=1\linewidth]{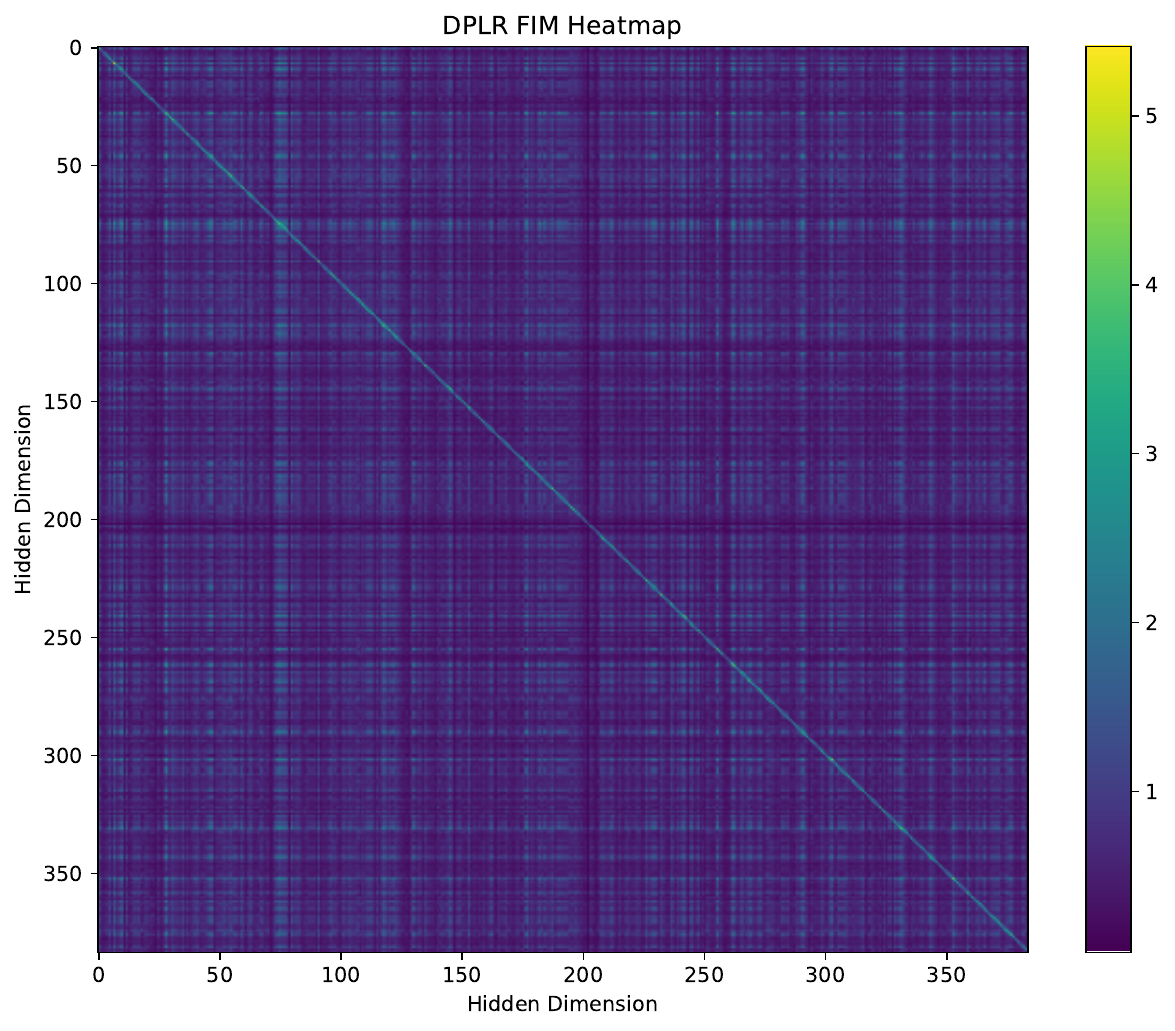}
        \caption{DPLR-FIM}
        \label{fig:fim_sub4}
    \end{subfigure}
    \caption{Illustration on the heatmap of FIM for the class token. (a) shows the complete FIM, where the diagonal elements are notably prominent, and the off-diagonal elements are also non-negligible. (b) and (c) display the diagonal approximation and low-rank approximation on FIM, respectively, both of which omit critical information from the complete FIM. (d) shows the proposed diagonal plus low-rank approximation on FIM, which clearly achieves improvement over compared approaches.}
    \label{fig:fim}
\end{figure}

By leveraging the powerful self-attention mechanism \cite{gpt3,bert}, Vision Transformers (ViTs) have recently established themselves as a fundamental architecture in the computer vision community \cite{vit,detr,transforming,Octr,Cat-det}, challenging the longstanding dominance of Convolutional Neural Networks (CNNs) \cite{alexnet,vgg,resnet}. However, the improved performance comes at the cost of larger model sizes, more parameters, and increased computational overhead. In order to facilitate the deployment of ViTs on resource-constrained hardware, substantial research efforts have been devoted to developing model compression techniques.

Model quantization, as one of the most effective model compression methods, aims to convert floating-point weights or activations to low-bit integers, thereby reducing memory consumption and computational cost \cite{whitepaper}. However, this process inevitably introduces quantization loss, leading to accuracy degradation. Conventional approaches employ Quantization-Aware Training (QAT) \cite{learnstep,packqvit,qvit} to recover performance through end-to-end retraining. Despite the significantly improved accuracy, QAT methods often requires training on the entire pretraining dataset, incurring prohibitive computational costs. Consequently, Post-Training Quantization (PTQ) methods, which calibrate quantization parameters on a small unlabeled data, have emerged as a promising alternative \cite{adaround,BRECQ,qdrop}.

While current PTQ methods have delivered promising performance for CNNs, they suffer from significant performance degradation when applied to ViTs, primarily due to the imbalanced and asymmetric activation distributions of ViTs. Although numerous studies have developed specialized quantizers to accommodate the unique structure and distribution of ViTs, these approaches still underperform at low bit-widths \cite{adalog,RepQViT,PTQ4ViT}. This propels us to examine the fundamental aspects of existing PTQ methods, and observe that most of them rely on a loss function to guide parameter tuning for minimizing quantization error. Consequently, the formulation and measurement of quantization loss substantially influence the performance of quantized models.  

Motivated by our observation that many existing methods essentially employ Hessian-guided quantization loss \citep{regptq,BRECQ,PTQ4ViT,APQViT}, we investigate the construction of the foundational quantization loss through Fisher Information Matrix (FIM) approximation in this work. Considering the computational intractability of exact Hessian computation, we explore viable approximation strategies. The existing diagonal approximation method, introduced by BRECQ \citep{BRECQ}, serves as a candidate, initially replacing the Hessian with FIM, followed by a diagonal approximation.

\begin{figure*}[t]
    \centering
    \includegraphics[width=1\linewidth]{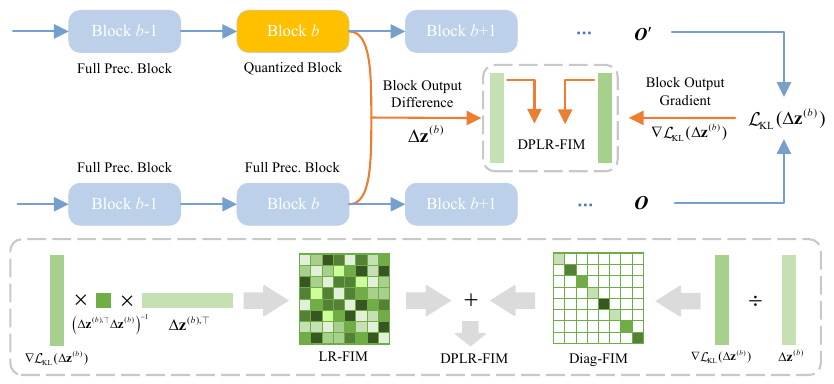}
    \caption{Framework overview of the proposed FIMA-Q method. We follow the block-wise quantization pipeline. For each block, we compute the difference between outputs before and after quantization, perform forward propagation through the rest of the networks, calculate the KL divergence, and conduct backward propagation to obtain the gradients. Based on the output difference and gradients, we compute the low-rank approximation and diagonal approximation on FIM respectively, and combine them to formulate the DPLR-FIM loss for quantization reconstruction.}
    \label{fig:overview}
\end{figure*}

However, our in-depth analysis of FIM approximation reveal two critical findings: 1) As validated in Table~\ref{tab:diff_a}, the diagonal approximation proposed by BRECQ achieves inferior performance compared to simple mean square error (MSE) loss on ViTs. This is particularly surprising since the MSE loss, being agnostic to task-specific objectives, should theoretically underperform Hessian-based methods in case of block-wise quantization. 2) As shown in Fig.~\ref{fig:fim} (a), the visualization of the complete FIM for the class token implies that both the diagonal elements and the off-diagonal inter-token correlations of FIM are important. This finding suggests that the prevailing diagonal approximation approach discards potentially valuable off-diagonal components that may be crucial for maintaining the performance.

To address these issues, we conduct a rigorous analysis of FIM approximation. Our investigation challenges the hypothesis of BRECQ, revealing that FIM is linearly proportional to the gradient of the Kullback-Leibler (KL) divergence, instead of the square of the gradient as BRECQ posits. Building upon this insight, we firstly propose a rank-one approximation method that preserves critical off-diagonal correlations and extend it to a low-rank approximation. Subsequently, inspired by the diagonal plus low-rank approximation, we further incorporate the diagnal and low-rank components, formulating a more concise estimation on FIM, as displayed in \cref{fig:fim} (d).

The main contributions of this paper are summarized in the following three aspects:

1) We conduct a thorough analysis on FIM approximation adopted by the prevailing Hessian-guided loss for post-training quantization. By exploring the relationship between KL divergence and FIM, we reveal that FIM is linearly proportional to the gradient of KL divergence.

2) We propose a novel rank-one approximation method for FIM and further extend it to low-rank approximation. By combining with the diagonal approximation, we formulate a new loss for quantization reconstruction.

3) We extensively evaluate the performance of our proposed method on image classification and object detection across various ViT-based network architectures. The experimental results demonstrate that our method, with a simple uniform quantizer, significantly outperforms the state-of-the-art approaches that adopt specialized quantizers, especially in the case of low-bit quantization.

\section{Related Work}
\label{sec:related}

Current quantization approaches are primarily divided into two categories: Quantization Aware Training (QAT) and Post Training Quantization (PTQ). QAT \citep{packqvit,qvit,ivit,learnstep} integrates quantization directly into the training process, which requires a tremendous amount of training data, computational resources, and time. In contrast, PTQ  requires only a smaller dataset and lower costs for the fine-tuning process. Currently, numerous PTQ methods have demonstrated promising performance when applied to CNNs. AdaRound \citep{adaround} pioneers the adaptive quantitative rounding strategy, bringing revolutionary progress in low-bit quantization. BRECQ \citep{BRECQ} introduces the block reconstruction framework with a Hessian quantization loss. Successively, QDrop \citep{qdrop} develops the randomly dropping quantization.

Due to the unique structure and properties of ViTs, the straightforward application of PTQ methods originally designed for CNNs \citep{adaround,BRECQ} does not perform well on ViTs.The problem can be summed up in two aspects: 1) inter-channel variation in layernorm and 2) non-uniform distribution in softmax and GELU. FQ-ViT \citep{FQViT} designed Power-of-Two Factor and \(\log 2\) quantizer to solve both problems first. To solve the inter-channel variation, many studies are dedicated to identifying suitable group quantization techniques, such as PEG \citep{PEG} and IGQ-ViT \citep{igqvit}. RepQ-ViT \citep{RepQViT} introduced a scale reparameterization technique, enabling the application of layer-wise quantizers. Meanwhile, to further match distributions of activation values after softmax and GELU, PTQ4ViT \citep{PTQ4ViT} proposed twin uniform quantization. Building on PTQ4ViT, APQ-ViT \citep{APQViT} worked on preserving the Matthew effect of post-Softmax activations and presented a unified Bottom-elimination Blockwise Calibration. Furthermore, AdaLog \citep{adalog} developed a non-uniform quantizer to adaptively select the logarithm base, further improving the accuracy. From an alternative perspective, we observe that many works \citep{regptq,BRECQ,PTQ4ViT,APQViT} have employed the hessian matrix as an importance metric. In this paper, we focus on a more accurate approximation of the Hessian matrix and more effective Hessian-guided quantization loss.
\section{The Proposed Method}
\label{sec:method}

Fig.~\ref{fig:overview} illustrates the framework of the proposed FIMA-Q method. Basically, FIMA-Q is built upon the QDrop \cite{qdrop} framework. Specifically, we reconstruct each block in the ViTs individually. Apart from replacing the MSE loss in QDrop with the DPLR-FIM based quantization loss proposed in \cref{sec:appfim}, we maintain the other steps as utilized in QDrop. The overall pipeline of the FIMA-Q for a certain Transformer block is summarized in Algorithm \ref{alg:overview}.

\begin{algorithm}[t]
\caption{Pipeline of Block-wise FIMA-Q.}
\label{alg:overview}
\begin{algorithmic}[1]
    \Statex \textbf{Input:} Full-precision model \(\mathcal{M}\), full-precision block \(\mathcal{B}_{\mathrm{full}}\), quantized block \(\mathcal{B}_{\mathrm{quant}}\), calibration data \(\mathcal{D}_{\mathrm{calib}}\), rank \(r\), maximal iteration number \(\mathrm{max\_iter}\), and iteration interval \(x\).
    \Statex \textbf{Output:} The optimized quantized block \(\mathcal{B}_{\mathrm{quant}}.\)
    \Statex \noindent \textcolor[RGB]{82,147,141}{ \# Calculate Rank-One FIM: }
    \State Generate the raw input \(\bm{X}_{\mathrm{raw}}\) and output \(\mathbf{z}^{(b)}\) of \(\mathcal{B}_{\mathrm{full}}\), the quantized input \(\bm{X}_{\mathrm{quant}}\) of \(\mathcal{B}_{\mathrm{quant}}\) based on \(\mathcal{D}_{\mathrm{calib}}\).
    \State Generate the model output \(O_{\mathrm{raw}}\) and \(O_{\mathrm{quant}}\) by performing forward propagations through the network starting from \(\mathcal{B}_{\mathrm{raw}}\) and \(\mathcal{B}_{\mathrm{quant}}\) based on \(\bm{X}_{\mathrm{raw}}\), respectively.
    \State Calculate the perturbation \(\Delta\mathbf{z}^{(b)}\) and the loss \(\mathcal{L}_{\mathrm{KL}}(\Delta\mathbf{z}^{(b)})\) based on \cref{eq:kl_fim}, and perform backward propagation to compute the gradient \(\nabla \mathcal{L}_{\mathrm{KL}}(\Delta\mathbf{z}^{(b)})\).
    \Statex \noindent \textcolor[RGB]{82,147,141}{ \# Perform Quantization Reconstruction: }
    \State Initialize the current rank \(k=1\), and the next data iteration \(y=x\).
    \For{$i = 1$, $\cdots$, \(\mathrm{max\_iter}\)}
        \If{\(k<r\) \textbf{and} \(i=y\)}
            \State \parbox[t]{\dimexpr\linewidth-\algorithmicindent-\algorithmicindent}{Calculate \(\Delta\mathbf{z}^{(b)'}\) and \(\nabla \mathcal{L}_{\mathrm{KL}}(\Delta\mathbf{z}^{(b)'})\).}
            \State Calculate \(B\) in \cref{eq:abc}.
            \State Set \(k := k + 1\) and \(y := y + x\).
        \EndIf
        \State Calculate \(\mathcal{L}_{\mathrm{DPLR}}\) with \cref{eq:dplr} and perform BP.
        \State \parbox[t]{\dimexpr\linewidth-\algorithmicindent}{Update all the AdaRound weights by \cite{adaround} and activation scaling factors in \(\mathcal{B}_{\mathrm{quant}}\).}
    \EndFor
\end{algorithmic}
\end{algorithm}

\subsection{Preliminaries}
BRECQ introduces a Hessian-guided quantization loss that models quantization as output perturbation during block-wise quantization. Through Taylor expansion, it specifically minimizes the second-order term as below:

\begin{equation} \label{eq:taylor}
    \min \mathbb{E}\left[-\Delta\mathbf{z}^{(b)\top}\mathbf{H}^{(\mathbf{z}^{(b)})}\Delta\mathbf{z}^{(b)}\right] ,
\end{equation}
where \(\mathbf{z}^{(b)}\) is the output of the block, \(\Delta\mathbf{z}^{(b)}\) is the difference between the output before and after quantization, \(\mathbf{H}^{(\mathbf{z}^{(b)})}\) is the Hessian matrix of the task loss \(\mathcal{L}_{\mathrm{task}}\) w.r.t. \(\mathbf{z}^{(b)}\).

Due to the large size of the Hessian matrix, BRECQ employs a diagonal approximation. Specifically, BRECQ estimates the diagonal of the negative Hessian using squared gradient values:
\begin{equation}
    -\mathbf{H}^{(\mathbf{z}^{(b)})} \approx \mathrm{Diag}\left((\frac{\partial \mathcal{L}_{\mathrm{task}}}{\partial \mathbf{z}^{(b)}_1})^2, \cdots, (\frac{\partial \mathcal{L}_{\mathrm{task}}}{\partial \mathbf{z}^{(b)}_a})^2\right) .
\end{equation}

\vspace{0.1in}However, within the PTQ framework, the absence of labeled data prevents task loss from being directly computed. Therefore, BRECQ adopts the KL divergence between pre- and post-quantization distributions to approximate the task loss. When calculating the negative Hessian matrix, BRECQ utilizes several approximations: 1) utilizing the Fisher Information matrix (FIM) to approximate the negative Hessian matrix. 2) using the diagonal FIM to approximate the full FIM. 3) adopting squared gradients of the task loss to approximate the diagonal of the FIM. 4) applying the gradient of KL divergence instead of the task loss.

We analyze the validity of these approximations as follows. First, we demonstrate the soundness of approximating the negative Hessian with the Fisher information matrix (FIM) in Sec. \ref{sec:relation}. Second, while the diagonal approximation on FIM is generally acceptable, we argue it introduces unnecessary limitations and develops an improved approach in Sec. \ref{sec:appfim}. Third, we identify the squared gradient approximation utilized in BRECQ is inaccurate and provide elaborate analysis in Sec. \ref{sec:relation}. Finally, we show the KL-divergence gradient substitution is well-justified, and provide the proof in Sec. \ref{sec:proof_ass_4} of the \emph{supplementary material}.

\subsection{Relationship Between FIM and KL-Divergence} \label{sec:relation}

Our method fundamentally relies on replacing the Hessian in \cref{eq:taylor} with FIM for optimization. We therefore first establish the validity of this replacement.

Specifically, on the calibration set, the expected value of reconstruction loss is written as:
\begin{equation} \label{eq:avg_hessian}
    \mathbb{E}\left[\mathcal{L}_{\mathrm{calib}}\right] = - \mathbb{E}\left[\Delta\mathbf{z}^{(b)\top}\mathbf{H}^{(\mathbf{z}^{(b)})}\Delta\mathbf{z}^{(b)}\right].\\
\end{equation}
Since the Hessian matrix is a function of the sample and the full-precision block output \(\mathbf{z}^{(b)}\), we assume that it is independent of the perturbation \(\Delta\mathbf{z}^{(b)}\) caused by quantization. Therefore, we can deduce the following:
\begin{equation}
    \mathbb{E}\left[\mathcal{L}_{\mathrm{calib}}\right] = - \mathbb{E}\left[\Delta\mathbf{z}^{(b),\top}\mathbb{E}\left[\mathbf{H}^{(\mathbf{z}^{(b)})}\right]\Delta\mathbf{z}^{(b)}\right].
\end{equation}

\begin{theory} \label{theo:fim_hessian}
    When the expected gradient of the log-likelihood of model outputs becomes zero, FIM \(\bm{F}^{(\bm{z}^{(b)})}\) equals to the expected negative Hessian, i.e.,
    \begin{align}
        \mathbf{F}^{(\mathbf{z}^{(b)})} = \mathbb{E} \left[ \left(\nabla_{\mathbf{z}^{(b)}} \log p(y;\mathbf{z}^{(b)})\right) \left(\nabla_{\mathbf{z}^{(b)}} \log p(y;\mathbf{z}^{(b)})\right)^\top \right],
    \end{align}
    where \(\log p(y;\mathbf{z}^{(b)})\) is the log-likelihood function.
\end{theory}

We provide the detailed proof in Sec.~\ref{sec:proof_3_1} of the \emph{supplementary material}. Based on Theorem \ref{theo:fim_hessian} and \cref{eq:avg_hessian}, the target to optimize can be written as:
\begin{equation} \label{eq:new_tar}
    \min \mathbb{E}\left[ \Delta\mathbf{z}^{(b),\top}\mathbf{F}^{(\mathbf{z}^{(b)})}\Delta\mathbf{z}^{(b)} \right].
\end{equation}

Note that FIM is an inherent attribute of the network block. In BRECQ, directly approximating the diagonal elements of  FIM as the squared gradient is inaccurate. Actually, by definition, the following equation holds
\begin{align} \label{eq:diag_fim}
\begin{split}
    \mathrm{Diag}\left( \mathbf{F}^{(\mathbf{z}^{(b)})}\right) &= \mathbb{E} \left[ \left(\nabla_{\mathbf{z}^{(b)}} \log p(y;\mathbf{z}^{(b)})\right)^2 \right] \\
    &=  \mathbb{E} \left[ \left(\nabla_{\mathbf{z}^{(b)}} \log p(y;\mathbf{z}^{(b)})\right) \right]^2 \\
    &\ \ \ \ + \mathrm{Var}\left(\nabla_{\mathbf{z}^{(b)}} \log p(y;\mathbf{z}^{(b)})\right).
\end{split}
\end{align}
In BRECQ, FIM is regarded as a sample-dependent matrix, and the second variance term in \cref{eq:diag_fim} is omitted, resulting in larger approximation errors.

Thereafter, we derive the relationship between the KL-divergence and FIM to propose an improved approximation of FIM. The KL-divergence based loss is defined as:
\begin{align} \label{eq:kl_div}
\begin{split}
    \mathcal{L}_{\mathrm{KL}}(\Delta\mathbf{z}^{(b)}) 
    &= D_{\text{KL}}(p(y;\mathbf{z}^{(b)}) \| p(y;\mathbf{z}^{(b)}+\Delta\mathbf{z}^{(b)})) \\
    &= \sum_{x} p(y;\mathbf{z}^{(b)}) \log \frac{p(y;\mathbf{z}^{(b)})}{p(y;\mathbf{z}^{(b)}+\Delta\mathbf{z}^{(b)})}. 
\end{split}
\end{align}
Accordingly, $\mathcal{L}_{\mathrm{KL}}(\Delta\mathbf{z}^{(b)}) $ satisfies the following property. 

\begin{theory}
    When adopting the KL-divergence as the task loss, the following relationship holds:
    \begin{equation} \label{eq:kl_fim}
        \mathcal{L}_{\mathrm{KL}}(\Delta\mathbf{z}^{(b)}) = \frac{1}{2} \Delta\mathbf{z}^{(b),\top} \mathbf{F}^{(\mathbf{z}^{(b)})} \Delta\mathbf{z}^{(b)}.
    \end{equation}
\end{theory}

We provide the detailed proof in Sec.~\ref{sec:proof_3_2} of the \emph{supplementary material}. Differentiating both sides of \cref{eq:kl_fim} w.r.t. \(\Delta\mathbf{z}^{(b)}\) yields the following:
\begin{equation} \label{eq:begin}
    \nabla \mathcal{L}_{\mathrm{KL}}(\Delta\mathbf{z}^{(b)}) = \frac{\partial \mathcal{L}_{\mathrm{KL}}(\Delta\mathbf{z}^{(b)})}{\partial \Delta\mathbf{z}^{(b)}} = \mathbf{F}^{(\mathbf{z}^{(b)})} \Delta\mathbf{z}^{(b)}.
\end{equation}

By adopting the diagonal form approximation on FIM  employed in prior works, we can derive that
\begin{equation} \label{eq:diag_approx}
    \mathbf{F}_{\mathrm{Diag}}^{(\mathbf{z}^{(b)})} = \mathrm{Diag}\left( \frac{\nabla \mathcal{L}_{\mathrm{KL}}(\Delta\mathbf{z}^{(b)})_{1}}{\Delta\mathbf{z}_{1}^{(b)}}, \cdots, \frac{\nabla \mathcal{L}_{\mathrm{KL}}(\Delta\mathbf{z}^{(b)})_{a}}{\Delta\mathbf{z}_{a}^{(b)}}\right),
\end{equation}
which indicates that the diagonal elements of FIM is linearly correlated with the gradient of KL-divergence, rather than being related to the squared gradient as claimed in BRECQ.

\subsection{Improved Approximations of FIM} \label{sec:appfim}
In this section, we establish improved estimations on FIM. Concretely, for the $i-$th sample, we denote that 
\begin{equation}
    \nabla \mathcal{L}_{\mathrm{KL}}(\Delta\mathbf{z}^{(b, i)}) = \mathbf{F}^{(\mathbf{z}^{(b)})} \Delta\mathbf{z}^{(b, i)}.
\end{equation}
Therefore, $\nabla \mathcal{L}_{\mathrm{KL}}(\Delta\mathbf{z}^{(b)})$ and $\Delta\mathbf{z}^{(b)}$ are rewritten as
\begin{align} \label{eq:avg}
\begin{split}
    \nabla \mathcal{L}_{\mathrm{KL}}(\Delta\mathbf{z}^{(b)}) &= \sum_{i=1}^{n} \nabla \mathcal{L}_{\mathrm{KL}}(\Delta\mathbf{z}^{(b, i)}), \\
    \Delta\mathbf{z}^{(b)} &= \sum_{i=1}^{n} \Delta\mathbf{z}^{(b, i)}.
\end{split}
\end{align}
By assuming the independence between samples, \cref{eq:begin} holds when using \cref{eq:avg}. 

Our optimization problem now reduces to optimizing  \cref{eq:new_tar} under the constraints that FIM satisfies \cref{eq:begin} while being symmetry and positive definite. As the solution is not unique, we present four distinct quantization losses based on different approximation forms of FIM.

\textbf{1. Diagonal approximation.} By assuming FIM is diagonal and based on \cref{eq:diag_approx}, the loss is formulated as:
\begin{equation}
    \mathcal{L}_{\mathrm{diag}} =  \left(\frac{\nabla \mathcal{L}_{\mathrm{KL}}(\Delta\mathbf{z}^{(b)})}{\Delta\mathbf{z}^{(b)}}\right)^\top \left(\Delta\mathbf{z}^{(b, i)}\right)^2.
\end{equation}

\textbf{2. Rank-one approximation.} By assuming FIM \(\mathbf{F}_{\mathrm{rank}-1}^{(\mathbf{z}^{(b)})} = \bm{u}\bm{u}^\top\) where \(\bm{u} \in \mathbb{R}^{a\times 1}\), we have
\begin{equation}
    \bm{u} = \frac{\nabla \mathcal{L}_{\mathrm{KL}}(\Delta\mathbf{z}^{(b)})}{\sqrt{\nabla \mathcal{L}_{\mathrm{KL}}(\Delta\mathbf{z}^{(b)})^\top \Delta\mathbf{z}^{(b, i)}}}.
\end{equation}

We refer to Sec.~\ref{sec:deriv_u} of the \emph{supplementary material} for details. Given the $i-$th sample, the loss is written as
\begin{equation}
    \mathcal{L}_{\mathrm{rank}-1} = \frac{\left(\Delta\mathbf{z}^{(b, i)\top} \nabla \mathcal{L}_{\mathrm{KL}}(\Delta\mathbf{z}^{(b)})\right)^2}{\nabla \mathcal{L}_{\mathrm{KL}}(\Delta\mathbf{z}^{(b)})^\top \Delta\mathbf{z}^{(b, i)}},
\end{equation}
which maintains an \(O(a)\) computational complexity. Unlike the diagonal approximation that treats block outputs as independent contributors to the task loss, our method captures collective dependencies among elements of the output, without introducing extra computational overhead.

\textbf{3. Low-rank approximation.} Before extending the rank-one approximation to rank-\(k\), we first establish the following corollary.
\begin{corollary} \label{theo:no_high_rank}
    When \(k>1\), it is typically difficult to find a low-rank matrix \(\bm{u}\in \mathbb{R}^{a\times k}\) such that \(\mathbf{F}^{(\mathbf{z}^{(b)})}=\bm{u}\bm{u}^\top\) satisfying \cref{eq:begin}.
\end{corollary}
We refer to Sec.~\ref{sec:proof_3_3} of the \emph{supplementary material} for a detailed description. Accordingly, we consider relaxing the constraint of symmetry to compute the low-rank FIM.

Specifically, by taking the Moore-Penrose inverse of \(\Delta\mathbf{z}^{(b)}\), the rank-\(k\) FIM is computer as
\begin{align}
\begin{split}
    \Delta\mathbf{z}^{(b),+} &= \left( \Delta\mathbf{z}^{(b)\top}\Delta\mathbf{z}^{(b)}\right)^{-1} \Delta\mathbf{z}^{(b)\top}, \\
    \mathbf{F}_{\mathrm{rank}-k}^{(\mathbf{z}^{(b)})} &= \nabla \mathcal{L}_{\mathrm{KL}}(\Delta\mathbf{z}^{(b)}) \Delta\mathbf{z}^{(b)+}.
\end{split}
\end{align}

Consequently, the optimization objective for the $i-$th sample is formed as
\begin{equation}
    \mathcal{L}_{\mathrm{rank}-k} = \Delta\mathbf{z}^{(b i)\top}\mathbf{F}^{(\mathbf{z}^{(b)})}\Delta\mathbf{z}^{(b, i)} = A\cdot B\cdot C,
\end{equation}
where 
\begin{align}\label{eq:abc}
\begin{split}
     A &= \Delta\mathbf{z}^{(b, i)\top} \nabla \mathcal{L}_{\mathrm{KL}}(\Delta\mathbf{z}^{(b)}), \\
    B &= \left( \Delta\mathbf{z}^{(b)\top}\Delta\mathbf{z}^{(b)}\right)^{-1}, \\
    C &= \Delta\mathbf{z}^{(b)\top} \Delta\mathbf{z}^{(b, i)}.
\end{split}
\end{align}

The computational complexity of computing $A$ and $C$ is \(O(ak)\), and $B$ can be preprocessed. Therefore, the overall complexity of computing \(\mathcal{L}_{\mathrm{recon}}\) is \(O(ak)\), which is acceptable for small rank \(k\).

It is worth noting that the low-rank approximation requires that the matrix \(\Delta\mathbf{z}^{(b)\top}\Delta\mathbf{z}^{(b)}\) is invertible, implying that the rank of \(\Delta\mathbf{z}^{(b)}\) should be \(k\). Therefore, we need to obtain \(k\) linearly independent perturbations and their corresponding gradients. We employ a progressive strategy to gradually increase \(k\) in practice. Specifically, we initialize with a rank of \(1\) and increase 1 rank by \(x\) iterations. For every \(x\) iteration, we recompute the averaged quantization perturbations and averaged gradients based on \cref{eq:avg}, and concatenate them with the previously computed perturbations \(\Delta\mathbf{z}^{(b)}\) and gradients \(\nabla \mathcal{L}_{\mathrm{KL}}(\Delta\mathbf{z}^{(b)})\).

\textbf{4. DPLR approximation.} As shown in Fig.~\ref{fig:fim}, the low-rank approximation of the FIM may degrade the impact of individual outputs on task loss. Thus, we combine the diagonal and low-rank approximate and obtain the following diagonal plus low-rank (DPLR) form:
\begin{equation}\label{eq:dplr-fim}
    \mathbf{F}_{\mathrm{DPLR}}^{(\mathbf{z}^{(b)})} = \alpha\cdot \mathbf{F}_{\mathrm{rank}-k}^{(\mathbf{z}^{(b)})} + (1-\alpha)\cdot \mathbf{F}_{\mathrm{diag}}^{(\mathbf{z}^{(b)})}.
\end{equation}

Accordingly, the optimization objective is finally formulated as below:
\begin{equation} \label{eq:dplr}
    \mathcal{L}_{\mathrm{DPLR}} = \alpha\cdot \mathcal{L}_{\mathrm{rank}-k} + (1-\alpha)\cdot  \mathcal{L}_{\mathrm{diag}}.
\end{equation}

In this work, we adopt Eq.~\eqref{eq:dplr} based on DPLR-FIM approximation in Eq.~\eqref{eq:dplr-fim} as the ultimate quantization loss.
\section{Experimental Results and Analysis}

\begin{table*}[ht]
\centering
\caption{
Comparison of the top-1 accuracy across various ViT-based models on ImageNet. ``*” denotes the results are based on our re-implementation as QDrop is originally designed for CNNs, ``OP” indicates whether the PTQ method is an optimization-based one. ``SQ” signifies that the method requires a specific quantizer rather than a standard uniform quantizer. The best results are highlighted in \textbf{bold}.
}
\begin{tabular}{ccccccccccc}
    \toprule
    \textbf{Method} & \textbf{OP} & \textbf{SQ} & \textbf{W/A} & \textbf{ViT-S} & \textbf{ViT-B} & \textbf{DeiT-T} & \textbf{DeiT-S} & \textbf{DeiT-B} & \textbf{Swin-S} & \textbf{Swin-B} \\
    \midrule
    Full-Prec & - & - & 32/32 & 81.39 & 84.54 & 72.71 & 79.85 & 81.80 & 83.23 & 85.27 \\
    \midrule
    PTQ4ViT \cite{PTQ4ViT} & $\times$ & $\checkmark$ & 3/3 & 0.10 & 0.10 & 3.50 & 0.10 & 31.06 & 28.69 & 20.13 \\
    RepQ-ViT \cite{RepQViT} & $\times$ & $\checkmark$ & 3/3 & 0.10 & 0.10 & 0.10 & 0.10 & 0.10 & 0.10 & 0.10 \\
    AdaLog \cite{adalog} & $\times$ & $\checkmark$ & 3/3 & 13.88 & 37.91 & 31.56 & 24.47 & 57.47 & 64.41 & 69.75 \\
    I\&S-ViT \cite{isvit} & $\checkmark$ & $\checkmark$ & 3/3 & 45.16 & 63.77 & 41.52 & 55.78 & 73.30 & 74.20 & 69.30 \\
    DopQ-ViT \cite{dopq} & $\checkmark$ & $\checkmark$ & 3/3 & 54.72 & 65.76 & 44.71 & 59.26 & 74.91 & 74.77 & 69.63 \\
    QDrop* \cite{qdrop} & $\checkmark$ & $\times$ & 3/3 & 41.05 & 74.75 & 46.88 & 50.95 &  72.97&  74.67&  76.57\\
    \rowcolor{gray!20} \textbf{FIMA-Q (Ours)}& $\checkmark$ & $\times$ & 3/3 & \textbf{64.09}&  \textbf{77.63}& \textbf{55.55}& \textbf{69.13}&  \textbf{76.54}&  \textbf{77.26}&  \textbf{78.82}\\
    \midrule
    PTQ4ViT \cite{PTQ4ViT} & $\times$ & $\checkmark$ & 4/4 & 42.57 & 30.69 & 36.96 & 34.08 & 64.39 & 76.09 & 74.02  \\
    APQ-ViT \cite{APQViT} & $\times$ & $\checkmark$ & 4/4 & 47.95 & 41.41 & 47.94 & 43.55 & 67.48 & 77.15 & 76.48 \\
    RepQ-ViT \cite{RepQViT}& $\times$ & $\checkmark$ & 4/4 & 65.05 & 68.48 & 57.43 & 69.03 & 75.61 & 79.45 & 78.32 \\
    ERQ \cite{erq} & $\times$ & $\checkmark$ & 4/4 &68.91 & 76.63 & 60.29 & 72.56 & 78.23 & 80.74 & 82.44 \\
    IGQ-ViT \cite{igqvit} & $\times$ & $\checkmark$ & 4/4 & 73.61 & 79.32 & 62.45 & 74.66 & 79.23 & 80.98 & 83.14 \\
    AdaLog \cite{adalog} & $\times$ & $\checkmark$ & 4/4 & 72.75 & 79.68 & 63.52 & 72.06 & 78.03 & 80.77 & 82.47 \\
    I\&S-ViT  \cite{isvit} & $\checkmark$ & $\checkmark$ & 4/4 & 74.87 & 80.07 & 65.21 & 75.81 & 79.97 & 81.17 & 82.60 \\
    DopQ-ViT \cite{dopq} & $\checkmark$ & $\checkmark$ & 4/4 & 75.69 & 80.95 & 65.54 & 75.84 & 80.13 & 81.71 & 83.34 \\
    QDrop* \cite{qdrop} & $\checkmark$ & $\times$ & 4/4 & 71.84 &  82.63& 65.27 &  72.64&  79.96&  81.21&  82.99\\
    OASQ \cite{oasq} & \checkmark & \(\times\) & 4/4 & 72.88 & 76.59 & 66.31 & 76.00 & 78.83 & 81.02 & 82.46 \\
    \rowcolor{gray!20} \textbf{FIMA-Q (Ours)}& $\checkmark$ & $\times$ & 4/4 & \textbf{76.68}&  \textbf{83.04}& \textbf{66.84}&  \textbf{76.87}&  \textbf{80.33}&  \textbf{81.82}&  \textbf{83.60}\\
   \midrule
    PTQ4ViT \cite{PTQ4ViT} & $\times$ & $\checkmark$ & 6/6& 78.63& 81.65& 69.68& 76.28& 80.25& 82.38& 84.01\\
    APQ-ViT \cite{APQViT} & $\times$ & $\checkmark$ & 6/6& 79.10& 82.21& 70.49& 77.76& 80.42& 82.67& 84.18\\
    NoisyQuant \cite{noisyquant} & $\times$ & $\checkmark$ & 6/6& 78.65 & 82.32 & - & 77.43 & 80.70 & 82.86 & 84.68 \\
    RepQ-ViT \cite{RepQViT}& $\times$ & $\checkmark$ & 6/6& 80.43& 83.62& 70.76& 78.90& 81.27& 82.79& 84.57\\
    IGQ-ViT \cite{igqvit} & $\times$ & $\checkmark$ & 6/6& 80.76& 83.77& 71.15& 79.28& 81.71& 82.86& 84.82\\
    AdaLog \cite{adalog} & $\times$ & $\checkmark$ & 6/6& \textbf{80.91}& 84.80& 71.38 & 79.39& 81.55& \textbf{83.19}& \textbf{85.09}\\
    I\&S-ViT  \cite{isvit} & $\checkmark$ & $\checkmark$ & 6/6& 80.43& 83.82& 70.85& 79.15& 81.68& 82.89& 84.94\\
    DopQ-ViT \cite{dopq} & $\checkmark$ & $\checkmark$ & 6/6& 80.52& 84.02& 71.17& 79.30& 81.69& 82.95& 84.97\\
    QDrop* \cite{qdrop} & $\checkmark$ & $\times$ & 6/6 & 79.59 &  84.68& 71.48&  79.15
&  81.69&  83.01&  84.94\\
    OASQ \cite{oasq} & \checkmark & \(\times\) & 6/6 & 80.60 & 83.81 & 71.52 & 79.50 & 81.72 & 82.76 & 84.91 \\
    \rowcolor{gray!20} \textbf{FIMA-Q (Ours)}& $\checkmark$ & $\times$ & 6/6 & 80.64&  \textbf{84.82}& 
\textbf{71.53}&  \textbf{79.52}&  \textbf{81.74}&  \textbf{83.19}&  85.01\\
   \bottomrule
\end{tabular}
\label{tab:ImageNet}
\end{table*}

In this section, we evaluate the performance of our method across distinct vision tasks including image classification, object detection and instance segmentation with various ViT-based architectures, by comparing to the state-of-the-art PTQ approaches. We also perform extensive ablation studies on the proposed components.  

\subsection{Experimental setup}
\noindent{\textbf{Datasets and Models.}~} For the image classification task, we adopt the ImageNet \cite{imagenet} dataset for validation, based on various vision transformer architectures, including ViT \cite{vit}, DeiT \cite{deit}, and Swin \cite{swin}. For object detection and instance segmentation, we evaluate on the COCO \cite{coco} dataset based on the representative Mask R-CNN \cite{maskrcnn} and Cascade R-CNN \cite{cascadercnn} models that utilize Swin Transformer as the backbone.

\textbf{Implementation Details.~} Similar to \cite{PTQ4ViT,RepQViT}, all the pretrained full-precision Vision Transformers for classification are obtained from the timm library\footnote{https://github.com/huggingface/pytorch-image-models}. The pretrained detection and segmentation models are obtained from MMDetection \cite{mmdetection}. By following the reconstruction-based PTQ methods \cite{BRECQ,qdrop,oasq,isvit}, we randomly select 1024 unlabeled images from ImageNet and 256 unlabeled images from COCO as the calibration sets for classification and object detection, respectively. We adopt the vanilla channel-wise uniform quantizers for weight quantization and layer-wise uniform quantizers for activation quantization, including post-Softmax activations. We fix the rank for the proposed DPLR-FIM module as \(k=15\).

As established in Theorem \ref{theo:fim_hessian}, the FIM can substitute for the Hessian matrix when the loss function corresponds to the log-likelihood, specifically the Cross-Entropy loss. In order to adapt this framework to the detection and segmentation tasks, we employ the classification head from the regional proposal network (RPN) as the task output. Specifically, we compute FIM using the KL divergence between the classification outputs from RPN by the original and quantized models. We apply block reconstruction exclusively to the backbone network and Feature Pyramid Network (FPN), while conducting calibration-only quantization for the RPN and Region of Interest (ROI) modules.

\begin{table*}[t]
\setlength\tabcolsep{7pt}
\centering
\caption{Quantization results (\%) on COCO for the object detection and instance segmentation tasks. ``*" and ``$\dagger$" indicate the results are based on our re-implementation or re-production using the officially released source code. The best results are highlighted in \textbf{bold}.}
\begin{tabular}{cccccccccccc}
   \toprule
   \multirow{3} * {\textbf{Method}} & \multirow{3} * {\textbf{Opt}} & \multirow{3} * {\textbf{SQ}} & \multirow{3} * {\textbf{W/A}} & \multicolumn{4}{c}{\textbf{Mask R-CNN}} & \multicolumn{4}{c}{\textbf{Cascade Mask R-CNN}} \\
   \cmidrule(lr){5-8}\cmidrule(lr){9-12}
   ~ & ~ & ~ & ~ & \multicolumn{2}{c}{\textbf{Swin-T}} & \multicolumn{2}{c}{\textbf{Swin-S}} & \multicolumn{2}{c}{\textbf{Swin-T}} & \multicolumn{2}{c}{\textbf{Swin-S}} \\
   ~ & ~ & ~ & ~ & AP\textsuperscript{b} & AP\textsuperscript{m} & AP\textsuperscript{b} & AP\textsuperscript{m} & AP\textsuperscript{b} & AP\textsuperscript{m} & AP\textsuperscript{b} & AP\textsuperscript{m} \\

   \midrule
   Full-Precision & - & - & 32/32  & 46.0 & 41.6 & 48.5 & 43.3 & 50.4 & 43.7 & 51.9 & 45.0 \\
   \midrule
    Baseline & $\times$ & \(\times\) & 4/4 & 34.6 & 34.2 & 40.8 & 38.6 & 45.9 & 40.2 & 47.9 & 41.6 \\
    PTQ4ViT \cite{PTQ4ViT} & \(\times\) & \checkmark & 4/4 & 6.9 & 7.0 & 26.7 & 26.6 & 14.7 &  13.5 & 0.5 & 0.5 \\
   APQ-ViT \cite{APQViT} & $\times$ & \checkmark & 4/4 & 23.7 & 22.6 & \textbf{44.7} & 40.1 & 27.2 & 24.4 & 47.7 & 41.1 \\
   RepQ-ViT \cite{RepQViT} & $\times$ & \checkmark & 4/4 & 36.1 & 36.0 & 44.2$_{42.7}\dagger$ & 40.2 & 47.0 & 41.1 & 49.3 & 43.1 \\
   ERQ \cite{erq} & $\times$ & \checkmark & 4/4 & 36.8 & 36.6 & 43.4 & 40.7 & 47.9 & 42.1 & 50.0 & 43.6 \\
   I\&S-ViT \cite{isvit} & \checkmark & \checkmark & 4/4 & 37.5 & 36.6 & 43.4 & 40.3 & 48.2 & 42.0 & 50.3 & 43.6 \\
   DopQ-ViT \cite{dopq} & \checkmark & \checkmark & 4/4 & 37.5 & 36.5 & 43.5 & 40.4 & 48.2 & 42.1 & 50.3 & \textbf{43.7} \\
   QDrop* \cite{qdrop} & \checkmark & \(\times\) & 4/4 & 36.2 & 35.4 & 41.6 & 39.2 & 47.0 & 41.3 & 49.0 & 42.5 \\
    \rowcolor{lightgray!45}\textbf{FIMA-Q (Ours)} & \checkmark & \(\times\) & 4/4 & \textbf{38.7} & \textbf{37.8} & 44.2 & \textbf{41.1} & \textbf{48.7} & \textbf{42.5} & \textbf{50.4} & \textbf{43.7} \\
   \bottomrule
\end{tabular}
\label{tab:detection}
\end{table*}
\subsection{Comparison to the State-of-the-art Approaches}
\textbf{Quantization Results for Classification on ImageNet.}
We first evaluate the performance of our method on the classification task on ImageNet in terms of top-1 accuracy, compared to the state-of-the-art PTQ approaches. To highlight the advantages of our method, we report results across various representative Transformer architectures, including ViT-S/B, DeiT-T/S and Swin-S/B, under 6, 4 and 3 bits. 

As displayed in \cref{tab:ImageNet}, FIMA-Q consistently promotes accuracy across different settings, with particularly significant gains in case of low-bit quantization. Concretely, for 6-bit quantization, our method achieves either superior or competitive accuracy compared to the second-best approaches. For 4-bit quantization, while most existing methods suffer substantial accuracy degradation, our method maintains robust performance in most cases. As for the challenging 3-bit quantization, the performance of competing methods degrades dramatically, some of which even exhibit extremely poor results, such as PTQ4ViT \citep{PTQ4ViT} and RepQ-ViT \citep{RepQViT}. In comparison, our proposed FIMA-Q achieves minimal accuracy loss, surpassing the second-best approach by 5.31\% on average. It is worth noting that many compared approaches such as PTQ4ViT, RepQ-ViT, AdaLog and DopQ-ViT attempt to boost the performance by designing specific quantizers (SQ), which however are generally difficult to implement on hardware in practice. In contrast, our method only utilizes a standard uniform quantizer, making it hardware-friendly while reaching superior accuracy. 

\noindent \textbf{Quantization Results for Object Detection and Instance Segmentation on COCO.}
As shown in \cref{tab:detection}, our method achieves the best results under W4/A4 in most cases for object detection and instance segmentation. To establish a fair comparison baseline, we adapt RepQ-ViT by replacing its specialized quantizer with a standard uniform quantizer. Notably, the methods that exclusively adopt uniform quantizers including Baseline and QDrop consistently underperform approaches employing specific quantizers. These results imply the inherent limitations of uniform quantization in handling activation distributions. Despite this challenge, our method can significantly mitigate the quantization error by leveraging the proposed FIM-based loss, reaching state-of-the-art performance even when using a standard uniform quantizer.

\subsection{Ablation Study}
As we mainly analyze the FIM-based reconstruction and propose a quantization loss based on DPLR-FIM approximation in this paper, we evaluate its effectiveness by comparing to distinct quantization losses and studying the influence of different ranks on the DPLR-FIM component.

\noindent \textbf{On distinct quantization losses.}~ We choose the MSE loss as the baseline method, and compare with the conventional representative Hessian-guided loss adopted by BRECQ \cite{BRECQ}, denoted by BRECQ-FIM. We also report the results of our method using the diagonal estimation, low-rank estimation as well as their combination, denoted by Diag-FIM, LR-FIM, and DPLR-FIM, respectively.

As shown in \cref{tab:diff_a}, despite incorporating second-order Hessian information, BRECQ-FIM generally performs worse than the conventional MSE loss, due to the inaccurate estimation on FIM. In contrast, based on our finding that FIM is linearly proportional to the KL divergence gradient rather than its square, both the proposed Diag-FIM and LR-FIM unleash the potential of Hessian-guided losses, substantially surpassing MSE and BRECQ-FIM. Their combination dubbed DPLR-FIM further boosts the performance, promoting the accuracy of MSE by 23.04\% and 18.18\% on ViT-S and DeiT-S under W3/A3 respectively, with an average improvement of 8.74\%.

\begin{table*}[ht]
\centering
\caption{Comparison of the top-1 accuracy (\%) by using different quantization losses across distinct Transformer architectures on ImageNet.}
\begin{tabular}{clccccccc}
    \toprule
    \textbf{\#Bits (W/A)} & \textbf{Method} &  \textbf{ViT-S} & \textbf{ViT-B} & \textbf{DeiT-T} & \textbf{DeiT-S} & \textbf{DeiT-B} & \textbf{Swin-S} & \textbf{Swin-B} \\
    \midrule
    \multirow{5}*{3/3} & MSE \cite{qdrop} & 41.05 & 74.75 & 46.88 & 50.95 & 72.97& 74.67& 76.57\\
    ~ & BRECQ-FIM \cite{BRECQ} & 14.65& 11.61& 36.57& 49.20& 58.76& 66.26& 70.15\\
    ~ & Diag-FIM \textbf{(Ours)} & 60.02&76.29& 55.54& 68.68& 76.32& 75.08& 77.87\\
    ~ & LR-FIM \textbf{(Ours)} & \textbf{64.09}& 77.46& 55.25& 68.91& 76.33& 76.03& 77.59\\
    ~ & \textbf{DPLR-FIM (Ours)} & \textbf{64.09}& \textbf{77.63}& \textbf{55.55}& \textbf{69.13}& \textbf{76.54}& \textbf{77.26}& \textbf{78.82}\\
    \midrule
    \multirow{5}*{4/4} & MSE \cite{qdrop} & 71.84 & 82.63& 65.27 & 72.64& 79.96& 81.21& 82.99\\
    ~ & BRECQ-FIM \cite{BRECQ} & 63.70& 76.26& 61.99& 72.52& 76.59& 80.52& 81.80\\
    ~ & Diag-FIM \textbf{(Ours)} & 75.88& 83.02& 66.81& 76.79& 80.19& 81.18& 83.35\\
    ~ & LR-FIM \textbf{(Ours)} & 76.47& \textbf{83.04}& 66.78& 76.66& 80.30& 81.60& 83.15\\
    ~ & \textbf{DPLR-FIM (Ours)} & \textbf{76.65}& \textbf{83.04}& \textbf{66.84}& \textbf{76.87}& \textbf{80.33}& \textbf{81.82}& \textbf{83.60}\\
\bottomrule
\end{tabular}
\label{tab:diff_a}
\end{table*}

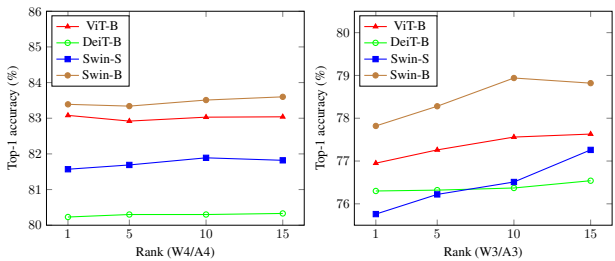
\begin{figure}[t]
\begin{tikzpicture}[scale=0.5]
\centering
\begin{axis}[xlabel=Rank (W4/A4),
        ylabel=Top-1 accuracy (\%),
        xtick={1, 5, 10, 15},
        ymin=80,ymax=86,
        ytick={ 80, 81, 82, 83, 84, 85, 86},
        legend style={at={(0.02,0.98)},anchor=north west}]
	\addplot[color=red,mark=triangle*,thick] coordinates{
		(1,83.08)
        (5,82.92)
        (10,83.03)
        (15,83.04)
	};
        \addlegendentry{ViT-B}
         \addplot[color=green,mark=o,thick] coordinates{
		(1,80.23)
        (5,80.30)
        (10,80.30)
        (15,80.33)
	};
        \addlegendentry{DeiT-B}
        \addplot[color=blue,mark=cube*,thick] coordinates{
		(1,81.57)
        (5,81.69)
        (10,81.89)
        (15,81.82)
	};
        \addlegendentry{Swin-S}
        \addplot[color=brown,mark=*,thick] coordinates{
		(1,83.39)
        (5,83.34)
        (10,83.51)
        (15,83.60)
	};
        \addlegendentry{Swin-B}
\end{axis}
\end{tikzpicture}
\begin{tikzpicture}[scale=0.5]
\centering
\begin{axis}[xlabel=Rank (W3/A3),
        ylabel=Top-1 accuracy (\%),
        xtick={1, 5, 10, 15},
        ymin=75.5,ymax=80.5,
        legend style={at={(0.02,0.98)},anchor=north west}]
        
	\addplot[color=red,mark=triangle*,thick] coordinates{
		(1,76.95)
        (5,77.26)
        (10,77.56)
        (15,77.63)
	};
        \addlegendentry{ViT-B}
        \addplot[color=green,mark=o,thick] coordinates{
		(1,76.30)
        (5,76.32)
        (10,76.37)
        (15,76.54)
	};
        \addlegendentry{DeiT-B}
        \addplot[color=blue,mark=cube*,thick] coordinates{
		(1,75.76)
        (5,76.22)
        (10,76.51)
        (15,77.26)
	};
        \addlegendentry{Swin-S}
        \addplot[color=brown,mark=*,thick] coordinates{
		(1,77.82)
        (5,78.28)
        (10,78.94)
        (15,78.82)
	};
        \addlegendentry{Swin-B}
\end{axis}
\end{tikzpicture}
\caption{Influence of the rank $k$ on the accuracy of DPLR-FIM on ImageNet.}
\label{fig:rank}
\end{figure}

\noindent \textbf{Influence of rank $k$ on DPLR-FIM.~} As the rank $k$ plays an important role in DPLR-FIM, we study its influence on both accuracy and efficiency of DPLR-FIM. As illustrated in Fig.~\ref{fig:rank}, our method performs steadily when using various values of $k$ at 4 bits. Under 3-bit quantization, the accuracy of DPLR-FIM improves with higher ranks, showing a gradual upward trend despite minor fluctuations.  
\begin{table}[t]

\setlength{\tabcolsep}{10pt}
\centering
\caption{The training time cost (GPU Minutes) using  different ranks \(k\) on a single Nvidia RTX 4090 GPU.}
\scalebox{0.90}{
\begin{tabular}{cccccc}
    \toprule
    \multirow{2} *{\textbf{Model}} &  \multirow{2} *{\textbf{Qdrop*}}&\multicolumn{4}{c}{\textbf{Ours}}\\
    \cmidrule(lr){3-6}
    ~ &~& \textbf{k=1} & \textbf{k=5} & \textbf{k=10} & \textbf{k=15} \\
    \midrule
    DeiT-T &  100&100 &  115 &  150 &  180 \\
    DeiT-S &  105&105 &  120 &  160 &  225 \\
    DeiT-B &  145&150 &  160 &  240 &  310 \\
    Swin-S &  160&165 &  235 &  360 &  420 \\
    Swin-B &  170&180 &  250 &  420 &  480 \\
\bottomrule
\end{tabular}
}
\label{tab:time}
\end{table}

In regards of efficiency, as described in Sec.~\ref{sec:appfim}, the computational complexity of the loss \(\mathcal{L}_{\mathrm{recon}}\) is \(O(ak)\), indicating that higher ranks linearly increase the reconstruction time cost. This necessitates balancing between accuracy and computational efficiency. To explore the trade-off, Table~\ref{tab:time} illustrates the impact of the rank \(k\) on the reconstruction time for 3-bit quantization. The results reveal that the training time cost slightly increases as $k$ becomes large for small models but grows substantially for large network architectures such as Swin-B. However, the overall training cost remains affordable, requiring less than 480 GPU minutes on a single Nvidia RX 4090 GPU.

\section{Conclusion}
\label{sec:conclusion}

In this paper, we propose a novel approach dubbed FIMA-Q for post-training quantization of Vision Transformers. Specifically, we propose a more accurate approximation method for FIM. Specifically, we first demonstrate that FIM is proportional to the gradient of KL-divergence, based on which we develop a novel estimation method dubbed DPLR-FIM by integrating both the diagonal and off-diagonal information. Extensive experimental results across distinct vision transformer architectures validate the effectiveness of FIMA-Q for various visual tasks. The results reveal that our method has achieved significant performance improvements in the cases of low-bit quantization, especially with an average improvement of 5.31\%, compared to the current state-of-the-art approaches in 3-bit quantization.

\section*{Acknowledgments}
This work was partly supported by the Beijing Municipal Science and Technology Project (No. Z231100010323002), the National Natural Science Foundation of China (No. 62202034), the Beijing Natural Science Foundation (No. 4242044), the Aeronautical Science Foundation of China
(No. 2023Z071051002), the Research Program of State Key Laboratory of Virtual Reality Technology and Systems, and the Fundamental Research Funds for the Central Universities.

{
    \small
    \bibliographystyle{ieeenat_fullname}
    \bibliography{main}
}

\clearpage
\newpage
\appendix
\renewcommand{\thetable}{\Alph{table}}
\setcounter{table}{0}

\renewcommand{\thefigure}{\Alph{figure}}
\setcounter{figure}{0}
\maketitlesupplementary

\section{Main Proofs} \label{sec:main_proof}
    The proofs of this section are based on a assumption that are generally adopted in Fisher Information Matrix (FIM), which is called the regularity condition \citep{fisher1922mathematical}. 
    
    \textbf{Regularity Condition:} Supposing that \(f\) is the likelihood with parameter \( \theta \), and \(T(x)\) is a function independent of the differentiation parameter \( \theta \), then : 
    \begin{equation}
      \frac{\partial}{\partial \theta} \int_{\mathbb{R}} T(x)f(x;\theta) \mathrm{d}x = \int_{\mathbb{R}} T(x)\frac{\partial}{\partial \theta} f(x;\theta) \mathrm{d}x. \\
    \end{equation}
\subsection{Proof of Approximation 4} \label{sec:proof_ass_4}

We believe that BRECQ adopts the gradient of KL divergence instead of the task loss gradient is based on the following theorems.

\begin{theory}
    When the model's output distribution matches the true data distribution, the Hessian matrix of the KL divergence after a small perturbation of the model is exactly equal to the expectation of the Hessian matrix of the model's likelihood function.
\end{theory}

\begin{proof}
    The Hessian matrix of the model's likelihood function is defined as:
    \begin{equation}
        \bm{H}(\theta) \triangleq \frac{\partial^2}{\partial \theta^2} \log f(X; \theta).
    \end{equation}
    As mentioned in Theorem \ref{theo:fim_hessian}, when the assumption that the model's output distribution matches the true data distribution is satisfied, the expectation of the Hessian matrix is equal to the negative Fisher Information Matrix.
    
    We adopt the integral form of the KL divergence to derive the KL divergence after a small perturbation of the model. Assume the output distribution of the model is \(p(x) = f(x;\theta)\), the output after perturbation is \(q(x) = f(x; \theta')\), where \(\theta'\) is a small perturbation w.r.t \(\theta\):
    \begin{equation}
    D_{\mathrm{KL}}(p \| q) = \int_{\mathbb{R}} f(x; \theta) \log \frac{f(x; \theta)}{f(x; \theta')} \, \mathrm{d}x.
    \end{equation}
    Similar to the definition of FIM \citep{fisher1922mathematical}, under the regularity condition, the Hessian matrix of KL divergence can be written as:
    \begin{equation}
    \begin{split}
        \frac{\partial^2}{\partial \theta_i' \partial \theta_j'} D_{\mathrm{KL}}(p \| q) 
        &= - \int_{\mathbb{R}} f(x; \theta) \left( \frac{\partial^2 \log f(x; \theta')}{\partial \theta_i' \partial \theta_j'} \right) \, \mathrm{d}x. \\
    \end{split}
    \end{equation}
    It can be observed that when \(f(x;\theta)\) matches the true data distribution, it can be regarded as the probability density function of the true data distribution. Thus, the Hessian matrix of KL divergence is equal to the expectation of the Hessian matrix of the log-likelihood of \(f(x, \theta')\). When \(\theta\) and \(\theta'\) are sufficiently close, the Hessian matrix of the KL divergence is the expectation of the Hessian matrix of the model's log-likelihood function.
\end{proof}

\subsection{Proof of Theorem 3.1} \label{sec:proof_3_1}

In order to prove Theorem 3.1, we begin with the definition of the score function.

\begin{definition}
    The Fisher Information Matrix is defined as the variance of the score function, where the score function is the gradient of the log-likelihood function.
\end{definition}

\begin{theory} \label{theo:score_exp}
    When the model's output distribution matches the true distribution, the expected value of the score function becomes \(0\).
\end{theory}

\begin{proof}
According to the definition of the score function {\citep{fisher1922mathematical}}, we have:
\begin{equation}
\begin{split}
\mathbb{E}\left[ \left. \frac{\partial}{\partial \theta} \log f(X;\theta) \right| \theta \right] &= \int_{\mathbb{R}} \frac{\frac{\partial}{\partial \theta} f(x;\theta)}{f(x;\theta)} f(x;\theta) \mathrm{d}x \\
&= \frac{\partial}{\partial \theta} \int_{\mathbb{R}} f(x;\theta) \mathrm{d}x \\
&= \frac{\partial}{\partial \theta} 1 \\
&= 0,
\end{split}
\end{equation}
where the likelihood function \( f(X;\theta) \) denotes the probability of the model output random variable 
 \(X\), and \( f(x;\theta) \) denotes the probability density of \(X\) at \(x\). This equation holds if and only if the output distribution matches the true distribution, allowing the use of \( f(x;\theta) \) as the probability density function for integration.
\end{proof}

Based on the conclusion above, we prove Theorem 3.1 as follows.

\begin{proof}
    Based on the definition of the Fisher Information Matrix and the definition of variance \(D(X) = E(X^2) - E^2(X)\), when the expected gradient of the log-likelihood is \(0\), we have:
    \begin{equation}
        \mathbf{F}(\theta) = \mathbb{D}\left[ \left. \frac{\partial}{\partial \theta} \log f(X;\theta) \right| \theta \right] = \mathbb{E}\left[ \left( \frac{\partial}{\partial \theta} \log f(X;\theta) \right)^2 \middle| \theta \right].
    \end{equation}
    The second derivative of the log-likelihood function with respect to the parameters (\emph{i.e.}, the Hessian matrix) is:
    \begin{equation} \label{eq:hessian_is_fisher}
    \begin{split}
        & \ \ \ \ \  \frac{\partial^2}{\partial \theta^2} \log f(X;\theta) = \frac{\partial}{\partial \theta} \frac{\frac{\partial}{\partial \theta} f(X;\theta)}{f(X;\theta)} \\
        &= \frac{\left(\frac{\partial^2}{\partial \theta^2} f(X;\theta)\right) \cdot f(X;\theta) - \left( \frac{\partial}{\partial \theta} f(X;\theta) \right)^2}{f(X;\theta)^2} \\
        &= \frac{\frac{\partial^2}{\partial \theta^2} f(X;\theta)}{f(X;\theta)} - \left(\frac{\frac{\partial}{\partial \theta} f(X;\theta)}{f(X;\theta)}\right)^2 \\
        &= \frac{\frac{\partial^2}{\partial \theta^2} f(X;\theta)}{f(X;\theta)} - \left(\frac{\partial}{\partial \theta} \log f(X;\theta)\right)^2,
    \end{split}
    \end{equation}
    where the second term is the definition of FIM, and under the regularity condition, the expectation of the first term is \(0\):
    \begin{equation}
    \begin{split}
        \mathbb{E}\left[ \left. \frac{\frac{\partial^2}{\partial \theta^2} f(X;\theta)}{f(X;\theta)} \right| \theta \right] 
        &= \int_{\mathbb{R}} \frac{\frac{\partial^2}{\partial \theta^2} f(x;\theta)}{f(x;\theta)} f(x;\theta) \mathrm{d}x \\
        &= \frac{\partial^2}{\partial \theta^2} \int_{\mathbb{R}} f(x;\theta) \mathrm{d}x = 0.
    \end{split}
    \end{equation}
    Therefore, when the model's output distribution matches the true distribution, the Fisher Information Matrix is equivalent to the expectation of the negative second derivative of the log-likelihood function, i.e., the expectation of the Hessian matrix of the negative log-likelihood function.
\end{proof}

\subsection{Proof of Theorem 3.2} \label{sec:proof_3_2}

\begin{proof}
In the context of block-wise post-training quantization, we regard the KL divergence as a function of the perturbation to the block output:
\begin{align}
\begin{split}
    \mathcal{L}_{\mathrm{KL}}(\Delta\mathbf{z}^{(b)})
    &= D_{\text{KL}}(p(x;\mathbf{z}^{(b)}) \| p(x;\mathbf{z}^{(b)}+\Delta\mathbf{z}^{(b)})) \\
    &= \int_{\mathbb{R}} p(x;\mathbf{z}^{(b)}) \log \frac{p(x;\mathbf{z}^{(b)})}{p(x;\mathbf{z}^{(b)}+\Delta\mathbf{z}^{(b)})} \\
    &= \int_{\mathbb{R}} p(x;\mathbf{z}^{(b)}) \log p(x;\mathbf{z}^{(b)}) \\
    &\ \ \ \ - \int_{\mathbb{R}} p(x;\mathbf{z}^{(b)}) \log p(x;\mathbf{z}^{(b)}+\Delta\mathbf{z}^{(b)}).
\end{split}
\end{align}
We perform a second order Taylor expansion to \( \log p(x;\mathbf{z}^{(b)}+\Delta\mathbf{z}^{(b)}) \) as below:
\begin{align}
\begin{split}
    & \ \ \ \log p(x;\mathbf{z}^{(b)}+\Delta\mathbf{z}^{(b)}) \\
    &= \log p(x;\mathbf{z}^{(b)}) + \nabla_{\mathbf{z}^{(b)}} \log p(x;\mathbf{z}^{(b)})^\top \Delta\mathbf{z}^{(b)} \\
    &\ \ \ \ +\frac{1}{2} \Delta\mathbf{z}^{(b)\top} \nabla_{\mathbf{z}^{(b)}}^2 \log p(x;\mathbf{z}^{(b)}) \Delta\mathbf{z}^{(b)}
\end{split}
\end{align}
Thus, we have
\begin{equation} \label{eq:s1_s2}
    \mathcal{L}_{\mathrm{KL}}(\Delta\mathbf{z}^{(b)}) = - \Delta\mathbf{z}^{(b)\top} \cdot S_1 - \frac{1}{2} \Delta\mathbf{z}^{(b)\top} \cdot S_2 \cdot \Delta\mathbf{z}^{(b)},
\end{equation}
where
\begin{align}
\begin{split}
    S_1 &= \int_{\mathbb{R}} p(x;\mathbf{z}^{(b)}) \nabla_{\mathbf{z}^{(b)}} \log p(x;\mathbf{z}^{(b)}) \\
    S_2 &= \int_{\mathbb{R}} p(x;\mathbf{z}^{(b)}) \nabla_{\mathbf{z}^{(b)}}^2 \log p(x;\mathbf{z}^{(b)}).
\end{split}
\end{align}

According to the properties of logarithmic function differentiation, we can deduce the following:
\begin{equation} \label{eq:gl}
    \nabla_{\mathbf{z}^{(b)}} p(x;\mathbf{z}^{(b)}) = p(x;\mathbf{z}^{(b)}) \nabla_{\mathbf{z}^{(b)}} \log p(x;\mathbf{z}^{(b)}).
\end{equation}
Thus, under the regularity condition, we have:
\begin{align}
\begin{split}
    S_1
    &= \int_{\mathbb{R}} p(x;\mathbf{z}^{(b)}) \nabla_{\mathbf{z}^{(b)}} \log p(x;\mathbf{z}^{(b)}) \\
    &= \int_{\mathbb{R}} \nabla_{\mathbf{z}^{(b)}} p(x;\mathbf{z}^{(b)}) \\
    &= \nabla_{\mathbf{z}^{(b)}} \int_{\mathbb{R}} p(x;\mathbf{z}^{(b)}) \\
    &= \nabla_{\mathbf{z}^{(b)}} 1 \\
    &= 0.
\end{split}
\end{align}
For \(S_2\), according to \cref{eq:gl}, the following equations hold:
\begin{align}
\begin{split}
    &\ \ \ \ \ \nabla_{\mathbf{z}^{(b)}}^2 p(x;\mathbf{z}^{(b)}) \\
    &= \nabla_{\mathbf{z}^{(b)}} (p(x;\mathbf{z}^{(b)}) \nabla_{\mathbf{z}^{(b)}} \log p(x;\mathbf{z}^{(b)})) \\
    &= \nabla_{\mathbf{z}^{(b)}} p(x;\mathbf{z}^{(b)}) \nabla_{\mathbf{z}^{(b)}} \log p(x;\mathbf{z}^{(b)})^\top \\
    &\ \ \ \ + p(x;\mathbf{z}^{(b)}) \nabla_{\mathbf{z}^{(b)}}^2 \log p(x;\mathbf{z}^{(b)}) \\
    &= p(x;\mathbf{z}^{(b)}) \nabla_{\mathbf{z}^{(b)}} \log p(x;\mathbf{z}^{(b)}) \nabla_{\mathbf{z}^{(b)}} \log p(x;\mathbf{z}^{(b)})^\top \\
    &\ \ \ \ + p(x;\mathbf{z}^{(b)}) \nabla_{\mathbf{z}^{(b)}}^2 \log p(x;\mathbf{z}^{(b)}).
\end{split}
\end{align}
Therefore, \(S_2\) can be written as
\begin{align}
\begin{split}
    S_2 &= \int_{\mathbb{R}} p(x;\mathbf{z}^{(b)}) \nabla_{\mathbf{z}^{(b)}}^2 \log p(x;\mathbf{z}^{(b)}) \\
    &= \int_{\mathbb{R}} \nabla_{\mathbf{z}^{(b)}}^2 p(x;\mathbf{z}^{(b)}) \\
    &\ \ \ \ - \int_{\mathbb{R}} p(x;\mathbf{z}^{(b)}) \nabla_{\mathbf{z}^{(b)}} \log p(x;\mathbf{z}^{(b)}) \nabla_{\mathbf{z}^{(b)}} \log p(x;\mathbf{z}^{(b)})^\top.
\end{split}
\end{align}
Under the regularity condition, we can derive that
\begin{align}
\begin{split}
    \int_{\mathbb{R}} \nabla_{\mathbf{z}^{(b)}}^2 p(x;\mathbf{z}^{(b)}) &= \nabla_{\mathbf{z}^{(b)}}^2 \int_{\mathbb{R}} p(x;\mathbf{z}^{(b)}) \\
    &= \nabla_{\mathbf{z}^{(b)}}^2 1 \\
    &= 0.
\end{split}
\end{align}
Thus,
\begin{align}
\begin{split}
    S_2 &= - \int_{\mathbb{R}} p(x;\mathbf{z}^{(b)}) \nabla_{\mathbf{z}^{(b)}} \log p(x;\mathbf{z}^{(b)}) \nabla_{\mathbf{z}^{(b)}} \log p(x;\mathbf{z}^{(b)})^\top \\
    &= - \mathbf{F}^{(\mathbf{z}^{(b)})}.
\end{split}
\end{align}
By substituting \(S_1\) and \(S_2\) into \cref{eq:s1_s2}, we have:
\begin{equation}
    \mathcal{L}_{\mathrm{KL}}(\Delta\mathbf{z}^{(b)}) = \frac{1}{2} \Delta\mathbf{z}^{(b)\top} \mathbf{F}^{(\mathbf{z}^{(b)})} \Delta\mathbf{z}^{(b)}.
\end{equation}
\end{proof}

\subsection{Derivation of Eq. (15)} \label{sec:deriv_u}

Given
\begin{align}
    \mathbf{F}^{(\mathbf{z}^{(b)})} &= \bm{uu}^\top, \\
    \nabla \mathcal{L}_{\mathrm{KL}}(\Delta\mathbf{z}^{(b)}) &= \mathbf{F}^{(\mathbf{z}^{(b)})} \Delta\mathbf{z}^{(b)},
\end{align}
where \(\bm{u},\nabla \mathcal{L}_{\mathrm{KL}}(\Delta\mathbf{z}^{(b)}),\Delta\mathbf{z}^{(b)}\in \mathbb{R}^{a\times 1}\), we define a scalar \(\alpha = \bm{u}^\top\cdot \Delta\mathbf{z}^{(b)}\) such that
\begin{equation}
    \nabla \mathcal{L}_{\mathrm{KL}}(\Delta\mathbf{z}^{(b)}) = \alpha \bm{u}.
\end{equation}
Then, we can deduce the following
\begin{equation}
    \bm{u}^\top = \frac{\left(\nabla \mathcal{L}_{\mathrm{KL}}(\Delta\mathbf{z}^{(b)})\right)^\top}{\alpha}.
\end{equation}
Thus,
\begin{align}
    \alpha = \sqrt{\left(\nabla \mathcal{L}_{\mathrm{KL}}(\Delta\mathbf{z}^{(b)})\right)^\top \Delta\mathbf{z}^{(b)}},\\
    \bm{u} = \frac{\nabla \mathcal{L}_{\mathrm{KL}}(\Delta\mathbf{z}^{(b)})}{\sqrt{\left(\nabla \mathcal{L}_{\mathrm{KL}}(\Delta\mathbf{z}^{(b)})\right)^\top \Delta\mathbf{z}^{(b)}}}.
\end{align}

\subsection{Proof of Corollary 3.1} \label{sec:proof_3_3}

\begin{proof}
Given
\begin{align}
    \mathbf{F}^{(\mathbf{z}^{(b)})} &= \bm{uu}^\top, \\
    \nabla \mathcal{L}_{\mathrm{KL}}(\Delta\mathbf{z}^{(b)}) &= \mathbf{F}^{(\mathbf{z}^{(b)})} \Delta\mathbf{z}^{(b)},
\end{align}
we can deduce the following
\begin{equation} \label{eq:duichen}
    \Delta\mathbf{z}^{(b)\top} \nabla \mathcal{L}_{\mathrm{KL}}(\Delta\mathbf{z}^{(b)}) = \Delta\mathbf{z}^{(b)\top} \bm{uu}^\top \Delta\mathbf{z}^{(b)}.
\end{equation}
Since the right-hand side of \cref{eq:duichen} is a symmetric matrix, the left-hand side should also be symmetric:
\begin{equation} \label{eq:duichen2}
    \Delta\mathbf{z}^{(b)\top} \nabla \mathcal{L}_{\mathrm{KL}}(\Delta\mathbf{z}^{(b)}) = \left (  \nabla \mathcal{L}_{\mathrm{KL}}(\Delta\mathbf{z}^{(b)}) \right )^\top \Delta\mathbf{z}^{(b)}.
\end{equation}
When \(k=1\), both sides of \cref{eq:duichen2} are scalars, implying that \cref{eq:duichen2} naturally holds. When \(k>1\), since \(\Delta\mathbf{z}^{(b)}\) and \(\mathcal{L}_{\mathrm{KL}}(\Delta\mathbf{z}^{(b)})\) are not directly related, we cannot guarantee their symmetry.

As a consequence, it is generally difficult to find a \(\bm{u}\) such that \(\mathbf{F}^{(\mathbf{z}^{(b)})} = \bm{uu}^\top\) satisfying Eq. (\ref{eq:begin}) in most cases.
\end{proof}

\section{More Experiments}\label{sec:FIMA-app}
As both FIM approximation (FIMA) and reconstruction steps depend on calibration data, we separately evaluate their performance utilizing different numbers of samples. As shown in Table~\ref{tab:ablation_sample_fim}, the accuracy using FIMA increases as sample size grows, but is generally robust to the sample size. However, the reconstruction step is more sensitive to the number of calibration samples. 

\begin{table}[h]
\centering
\caption{Ablation results (\%) w.r.t. the samples size with W3/A3 on ImageNet.}
\scriptsize
\begin{tabular}{ccccccc}
   \toprule
   \multirow{2} * {\textbf{Sample Size}} & \multicolumn{3}{c}{\textbf{In FIMA Step}} & \multicolumn{3}{c}{\textbf{In Reconstruction Step}} \\
    \cmidrule(lr){2-4}\cmidrule(lr){5-7}
   ~ & ViT-S& DeiT-S& Swin-S & ViT-S& DeiT-S&Swin-S 
\\
   \midrule
   \textbf{128}& 63.52& 68.99& 77.10
& 49.64& 64.12&71.71
\\
 \textbf{256}& 63.18& 69.10&77.00 
& 56.02& 66.21&74.12
\\
 \textbf{512}& 63.61& 69.14&77.18
 & 60.45& 67.87&75.8
\\
 \textbf{1024}& 64.09& 69.13&77.26
& 64.09& 69.13&77.26
\\
    \bottomrule
\end{tabular}

\label{tab:ablation_sample_fim}
\end{table}

Since the Fisher Information Matrix (FIM) captures global information, its computation involves averaging over the sample dimension. Theoretically, a larger sample size leads to a more accurate approximation due to reduced sampling error. However, since the averaging process mitigates the impact of individual sample variations, the difference is not particularly significant. In fact, even using a single sample for approximation can still yield an acceptable level of accuracy. However, as shown in \cref{tab:ablation_sample_fim}, directly altering the overall sample size leads to a more substantial accuracy change, as the reconstruction process in Adaround is more sensitive to the number of samples.

\end{document}